\documentclass{amsart}  
\usepackage[english]{babel}
\usepackage[utf8]{inputenc}
\usepackage{algorithmic}
\usepackage{algorithm}
\usepackage{amsfonts}
\usepackage{amsmath}
\usepackage{amssymb}
\usepackage{amsthm}
\usepackage{enumerate}
\usepackage{epsfig}
\usepackage{graphicx}
\usepackage{amsaddr}
\usepackage{subfigure}
\usepackage{url}

\newtheorem{theorem}{Theorem}
\newtheorem{lemma}[theorem]{Lemma}

\newcommand{\pv}{\textbf{p}}
\newcommand{\qv}{\textbf{q}}

\newcommand{\uv}{\textbf{u}}
\newcommand{\vv}{\textbf{v}}

\newcommand{\sigmav}{\mbox{\boldmath$\sigma$}}

\newcommand{\Omegam}{\mbox{\boldmath$\Omega$}}

\newcommand{\Sigmam}{\mbox{\boldmath$\Sigma$}}

\newcommand{\diag}{\mbox{diag}}

\newcommand{\st}{\mbox{s.t.}}

\newcommand{\Am}{\textbf{A}}
\newcommand{\Bm}{\textbf{B}}
\newcommand{\Cm}{\textbf{C}}

\newcommand{\Imat}{\textbf{I}}

\newcommand{\Pm}{\textbf{P}}
\newcommand{\Qm}{\textbf{Q}}
\newcommand{\Rm}{\textbf{R}}

\newcommand{\Um}{\textbf{U}}
\newcommand{\Vm}{\textbf{V}}
\newcommand{\Wm}{\textbf{W}}
\newcommand{\Xm}{\textbf{X}}
\newcommand{\Ym}{\textbf{Y}}
\newcommand{\Zm}{\textbf{Z}}

\title{Online Matrix Completion Through\\Nuclear Norm Regularisation }

\author{Charanpal Dhanjal \and St\'ephan Cl\'{e}men\c{c}on}
\address[Charanpal Dhanjal]{Institut Mines-T\'{e}l\'{e}com; T\'{e}l\'{e}com ParisTech, CNRS LTCI, 46 rue Barrault, 75634 Paris Cedex 13, France}
\email[Charanpal Dhanjal]{\{charanpal.dhanjal, stephan.clemencon\}@telecom-paristech.fr}
\author{Romaric Gaudel}
\address[Romaric Gaudel]{Université Lille 3, Domaine Universitaire du Pont de Bois, 59653 Villeneuve d'Ascq Cedex, France}
\email[Romaric Gaudel]{romaric.gaudel@univ-lille3.fr}
\date{\today}

\begin{document}

\begin{abstract}
It is the main goal of this paper to propose a novel method to perform matrix completion \textit{on-line}. Motivated by a wide variety of applications, ranging from the design of recommender systems to sensor network localization through seismic data reconstruction, we consider the matrix completion problem when entries of the matrix of interest are observed \textit{gradually}. Precisely, we place ourselves in the situation where the predictive rule should be refined incrementally, rather than recomputed from scratch each time the sample of observed entries increases. The extension of existing matrix completion methods to the sequential prediction context is indeed a major issue in the Big Data era, and yet little addressed in the literature. The algorithm promoted in this article builds upon the {\sc Soft Impute} approach introduced in \cite{mazumder2010spectral}. The major novelty essentially arises from the use of a randomised technique for both computing and updating the \textit{Singular Value Decomposition (SVD)} involved in the algorithm. Though of disarming simplicity, the method proposed turns out to be very efficient, while requiring reduced computations. Several numerical experiments based on real datasets illustrating its performance are displayed, together with preliminary results giving it a theoretical basis.
\end{abstract}

\maketitle

\section{Introduction}
The task of finding unknown entries in a matrix given a sample of observed entries is known as the \emph{matrix completion} problem, see \cite{candes2009exact}. Stated in such a general manner, this corresponds to a wide variety of problems, including collaborative filtering, dimensionality reduction, image processing or multi-class classification for instance. In particular, it has recently received much attention in the area of \textit{recommender systems}, for e-commerce especially. In this context, users rate a selection of items (\textit{e.g.} books, movies, news) and then an algorithm predicts future still unobserved ratings based on the observed ones. The ratings can be represented through a matrix $\Xm \in \mathbb{R}^{m \times n}$, where the rows correspond to the users and the columns to the items: the $ij$-th entry is the rating, valued in $\{1, 2, \ldots, 5\}$ say, by user $i\in\{1,\;\ldots,\; m\}$ of item $j\in\{1,\;\ldots,\; n\}$ and only a sample of the entries of $\Xm$ are observed. The goal is to predict the rating for an unobserved entry $\Xm_{ij}$. In general, matrix completion is an impossible task, however under certain assumptions on $\Xm$ (typically $\Xm$ has low rank and has a certain number of observed entries)  \cite{candes2009exact, candes2010power, keshavan2010matrix}, the remaining entries of this matrix can be accurately recovered. 

A number of approaches have recently been proposed for matrix completion, the key idea of which is to solve the rank minimisation problem. As this optimisation is NP hard, one can use the \emph{nuclear norm}, or equivalently the sum of the singular values, as a surrogate which leads to a tractable problem, see \cite{candes2009exact, candes2010power, keshavan2010matrix, cai2010singular, candes2010matrix} for example. Most algorithms documented in the literature do not deal with common scenario of having an increasing amount of data, and yet developing methods capable to meet real-time constraints is of huge importance in the beginning Big Data era. Thus it is precisely this scenario which we address in this paper, known as \emph{incremental} or \emph{online learning}. In incremental matrix completion, given a sequence of incomplete matrices $\Xm_1, \ldots, \Xm_T$, possibly with different sizes, one wishes to complete each one without full recomputation at each iteration. 

Here we build upon the work of the \emph{Soft Impute} algorithm of \cite{mazumder2010spectral} for nuclear norm regularised matrix completion.  This work is simple to implement and analyse, scales to relatively large matrices and achieves competitive errors compared to state-of-the-art algorithms, however a bottleneck in the algorithm is the use of the Singular Value Decomposition (SVD, \cite{golub2012matrix}) of a large matrix at each iteration. We use recent work on the theory and practice of randomised SVDs \cite{halko2011finding} along with a novel updating method to improve the efficiency of matrix completion both in offline and online cases. In the offline case, we show that the algorithm can be efficiently and accurately implemented using only a single SVD of the input matrix and thereafter inexpensive updates of the SVD can be applied. In the online case, we provide an efficient path to updating the solution matrices under possibly high-rank perturbations of the input matrices. The randomised SVD can introduce errors into the final solution, and we give some theoretical and empirical insight into this error. The resulting randomised online learning algorithm is simple to implement, and readily parallelisable to make full use of modern computer architectures. Computational results on a toy dataset and several large real movie recommendations datasets shows the efficacy of the approach. 

In the following section, we present the main results in nuclear norm regularised matrix completion. Following on in Section \ref{sec:onlineCompletion} we introduce the online matrix completion algorithm and provide a preliminary theoretical analysis of it. Section \ref{sec:results} presents an empirical study of the resulting matrix completion approach and finally some concluding remarks are collected in Section \ref{sec:conclusion}. 

\section{Matrix Completion}\label{sec:matcompletion}

Consider again the matrix $\Xm \in \mathbb{R}^{m \times n}$ and a set of indices of the observed entries $\omega \subset \{1, \ldots, m\} \times \{1, \ldots, n\}$. A useful optimisation to consider in terms of the applications outlined above is that of the minimum rank subject to bounded errors on the observed entries, 
\begin{equation} 
 \begin{array}{c c} 
  \min & \mbox{rank}(\Zm) \\ 
\st & \sum_{i,j \in \omega} (\Xm_{ij} - \Zm_{ij})^2 \leq \delta,
 \end{array}\label{eqn:opt0}
\end{equation}
\noindent
for user defined $\delta \geq 0$, and some $\Zm \in \mathbb{R}^{m \times n}$, however this optimisation is NP hard. A number of ways to tackle this optimisation exist, for example using greedy selection \cite{shalev2011large, lee2010admira}. Another approach is to relax the rank term into a \emph{trace} or nuclear norm, i.e. solve  
\begin{equation} 
 \begin{array}{c c} 
  \min & \|\Zm\|_* \\ 
\st & \sum_{i,j \in \omega} (\Xm_{ij} - \Zm_{ij})^2 \leq \delta,
 \end{array}\label{eqn:opt1}
\end{equation}
\noindent 
which is a convex problem. The nuclear norm $\| \cdot \|_*$ is the sum of the singular values of a matrix  $\sum_{i=1}^r \sigma_i$, where $r$ is the rank of the matrix. The above can be reformulated in Lagrange form as follows: 
\begin{equation} \label{eqn:opt2}
 \min \frac{1}{2} \sum_{i,j \in \omega} (\Xm_{ij} - \Zm_{ij})^2  + \lambda \|\Zm\|_*,
\end{equation}
\noindent 
where $\lambda$ is a user-defined regularisation parameter. Another way of writing the above is in terms of the projection operator $P_\omega$: $\min \frac{1}{2} \|P_\omega(\Xm) - P_\omega(\Zm)\|^2_F  + \lambda \|\Zm\|_*$, where $P_\omega(\Xm)$ is the matrix whose $(i,j)$th entry is $\Xm_{ij}$ if $i,j \in \omega$ otherwise it is $0$, and $\|\Xm\|_F^2 = \sum_{ij}\Xm_{ij}^2$ is the Frobenius norm. Similarly,  $P_\omega^\bot(\Xm)$ is the matrix whose $(i,j)$th entry is $\Xm_{ij}$ if $i,j \notin \omega$ otherwise it is $0$. In \cite{candes2010power} it is shown that if $\Xm$ has entries sampled uniformly at random with $|\omega| \geq C p^{5/4} r \log p$, for a positive constant $C$ and with $p=\max(m, n)$, then one can recover all the entries of $\Xm$ with no error with high probability using trace norm minimisation.

In \cite{recht2011parallel} the nuclear norm penalised objective is approximated by writing the penalty in terms of the minimum Frobenius norm factorisation, and solving it using a parallel projected incremental gradient method. The resulting problem can be considered as a generalisation of Maximal-Margin Matrix Factorisation (MMMF, \cite{srebro2005maximum}). Note that the optimisation of Equation \eqref{eqn:opt2} can be considered a generalisation of the $\ell_1$ regularised least squares problem which is addressed in \cite{tibshirani1996regression}. As with $\ell_1$ versus $\ell_0$ linear regression, minimising the nuclear norm can outperform the rank minimised solution, as supported empirically in \cite{mazumder2010spectral}.  

Classical algorithms for solving semi-definite programs such as interior point methods are expensive for large datasets, and this has motivated a number of algorithms with better scaling.  In \cite{cai2010singular} the authors propose a Singular Value Thresholding (SVT) algorithm for the optimisation of Equation \eqref{eqn:opt1} in which $\delta = 0$.  In \cite{ma2011fixed} the authors use a similar approach based on Bregman iteration, and \cite{toh2010accelerated} uses an accelerated proximal gradient algorithm which gives an $\epsilon$-accurate solution in $\mathcal{O}(1/\sqrt{\epsilon})$ steps.  In \cite{jaggi2010simple} a variant of Equation \eqref{eqn:opt1} is solved in which there is an upper bound on the nuclear norm. The authors transform the problem into a convex one on positive semi-definite matrices. A  nuclear norm minimisation subject to linear and second order cone constraints is solved in \cite{liu2012implementable}, with an application to recommendation on a large movie rating dataset. The soft impute algorithm of \cite{mazumder2010spectral} is inspired by SVT, however unlike \cite{cai2010singular, ma2011fixed, ji2009accelerated} it does not require a step size parameter. Instead, soft impute is controlled using the regularisation parameter $\lambda$, and using warm restarts one can compute the complete regularisation path for model selection. The algorithm is shown to be competitive to SVT and MMMF and scalable to relatively large datasets. 

\subsection{Soft Impute}

Our novel online matrix completion scheme is based on soft impute, and hence we present the pseudo code of soft impute in Algorithm \ref{alg:softImpute}. As input, it takes a partially observed matrix $\Xm$ and a sequence of regularisation parameters $\lambda_1 > \lambda_2 > \ldots > \lambda_k$. The core part of the algorithm is the loop at Step \ref{step:while} which updates the current solution $\Zm_{j+1}$ and checks convergence with successive solutions. 

At step \ref{step:svd} one computes the SVD of $P_\omega(\Xm) - P_\omega(\Zm)$, which is the most expensive part of this algorithm. The SVD of $\Am \in \mathbb{R}^{m \times n}$ is the decomposition $\Am=\Pm \Sigmam \Qm^T$,  where $\Pm = [\pv_1, \ldots, \pv_r]$, $\Qm = [\qv_1, \ldots, \qv_r]$  are respective matrices whose columns are left and right singular vectors, and $\Sigmam = \diag(\sigma_1, \ldots, \sigma_r)$ is a diagonal matrix of singular values $\sigma_1 \geq \sigma_2 \geq, \ldots, \geq \sigma_r$, with $r = \min(m, n)$. One then applies the \emph{matrix shrinkage operator}, $S_{\lambda}(\Am) = \Pm \Sigmam_\lambda \Qm^T$, in which $\Sigmam_\lambda = \diag((\sigma_1-\lambda)_+, \ldots, (\sigma_r-\lambda)_+)$ where $t_+ = \max(t, 0)$. The idea behind this step is to converge to the stationary point of the objective of Equation \eqref{eqn:opt2}, which is the solution to $\Zm = S_\lambda(P_\omega(\Xm) - P_\omega^\bot(\Zm))$, the proof of which is given in \cite{mazumder2010spectral}. 
\begin{algorithm}
\caption{Pseudo-code for Soft Impute}
\begin{algorithmic}[1]
\REQUIRE Matrix $\Xm$, regularisation parameters $\lambda_1 > \lambda_2 > \ldots > \lambda_k$,  error threshold $\epsilon$  
\STATE $\Zm^{(1)} = 0$, $j=1$  
\FOR{$i = 1 \to k$}
\STATE $\gamma = \epsilon + 1$
\WHILE {$\gamma > \epsilon$} \label{step:while}
\STATE $\Zm^{(j+1)} \leftarrow S_{\lambda_i}(P_\omega(\Xm) + P^\bot_\omega(\Zm^{(j)}))$ \label{step:svd}
\STATE $\gamma = \frac{\|\Zm^{(j+1)} - \Zm^{(j)}\|^2_F}{\|\Zm^{(j)}\|_F^2}$ 
\STATE $j \leftarrow j + 1$
\ENDWHILE 
\STATE $\Zm_{\lambda_i} \leftarrow \Zm^{(j)}$ 
\ENDFOR
\RETURN Solutions $\Zm_{\lambda_1}, \ldots, \Zm_{\lambda_k}$
\end{algorithmic}\label{alg:softImpute}
\end{algorithm}
An important point about soft impute is that it appears to generate a dense matrix in step \ref{step:svd}, however the authors note that $\Ym = P_\omega(\Xm) + P^\bot_\omega(\Zm)$ can be written as the sum of a sparse term $P_\omega(\Xm) - P_\omega(\Zm)$ and a low rank term $\Zm$. Since the fundamental step in computing an SVD of $\Ym$ is matrix-vector multiplications $\Ym\uv$ and $\Ym^T\vv$, one can use this observation to improve the efficiency of the soft-thresholded SVD. First note that to compute $P_\omega(\Zm)$ requires $\mathcal{O}(|\omega|\tilde{r})$ operations using the SVD, in which $\Zm$ has a rank $\tilde{r} \ll m, n$. Furthermore a matrix-vector multiplication of  $\Ym$ can be found in order $\mathcal{O}((m+n)\tilde{r} + |\omega|)$ operations, the sum of the low-rank and sparse matrix multiplications. If there are $s$ singular vectors then the total computational cost is $\mathcal{O}((m+n)\tilde{r}s + |\omega|s)$. The final solution of $\Zm$ has rank $r \approx \tilde{r}$ and given we want to find approximately $r$ singular vectors this cost can be written as $\mathcal{O}((m+n)r^2t + |\omega|rt)$ assuming that $t$ iterations are required.

\section{Online Matrix Completion}\label{sec:onlineCompletion}

One disadvantage of Algorithm \ref{alg:softImpute} is that at each stage one must compute the SVD of a large matrix (in common with other SVT-based algorithms). Normally one uses a Krylov subspace method such Lanczos or Arnoldi (e.g. PROPACK \cite{larsen1998lanczos}) to compute the rank-$k$ SVD at cost $\mathcal{O}(kT_{mult} + (m+n)k^2)$ where $T_{mult}$ is the cost of a matrix-vector multiplication. In a sparse matrix, $T_{mult}$ is the number of nonzero elements $|\omega|$ in the matrix. A second disadvantage is that one must compute successive SVDs of a matrix $P_\omega(\Xm) + P^\bot_\omega(\Zm)$ ignoring previous computations of this SVD. One attempt to address the former point is given in \cite{zhao2011fast} which decomposes the input into a set of Kronecker products of several smaller matrices in conjunction with the algorithm of \cite{mazumder2010spectral}. This leads to two convex subproblems on smaller matrices, however a drawback of the approach is that one must specify the size of the decomposition in advance and the best choice is unknown a priori. 

Recently, randomised method for computing the SVD have been studied in the literature \cite{halko2011finding}. These approaches are competitive in terms of computational time with state-of-the-art Krylov methods, robust, well studied theoretically, and benefit from simple implementations and easy parallelisation. Whereas Lanczos and Arnoldi algorithms are numerically unstable, randomised algorithms are stable and come with performance guarantees that do not depend on the spectrum of the input matrix. The key idea of randomised algorithms is to project the rows onto a subspace which captures most of the ``action'' of the matrix. To illustrate the point, we recount an algorithm from \cite{halko2011finding} which is used in conjunction with kernel Principal Components Analysis (KPCA, \cite{sch98nonlinear}) in \cite{yun2011nystrom}. Algorithm \ref{alg:randomSVD} provides the associated pseudo-code. The purpose of the first three steps is to find a matrix $\Vm \in \mathbb{R}^{m \times (k+p)}$ such that the projection of $\Am$ onto $\Vm$ is a good approximation of $\Ym$. In other words, we hope to find $\Vm$ with orthogonal columns such that $\|\Am - \Vm\Vm^T\Am\|_2$ is small, where $\|\Am\|_2 =\sigma_{1}$ is the spectral norm. The above norm is minimised when the columns of $\Vm$ are made up of the $(k+p)$ largest left singular vectors of $\Am$. When $q = 0$, the matrix $\Ym$ is one whose columns are random samples from the range of $\Am$ under a rank-$(k+p)$ projection $\Omegam$. The columns of $\Omegam$ are likely to be linearly independent as are those of $\Ym$ which span much of the range of $\Am$ provided its range is not much larger than $(k+p)$. Hence, the resulting projection is orthogonalised to form $\Vm$ and then one need only find the SVD of the smaller matrix $\Bm = \Vm^T\Am$. When $q > 0$ the quality of $\Vm$ is improved when the spectrum of the data matrix decays slowly, as is often the case in matrix completion. Note that to reduce rounding errors one othogonalises the projected matrix before each multiplication with $\Am$ or $\Am^T$. The complexity of the complete approach is $\mathcal{O}((q+1) (k+p) T_{mult} + (k+p)^2(m + n))$. 
\begin{algorithm}
\caption{Randomised SVD \cite{halko2011finding}}
\begin{algorithmic}[1]
\REQUIRE Matrix $\Am \in \mathbb{R}^{m \times n}$, target rank $k$, oversampling projection vectors $p$, exponent $q$ 
\STATE Generate a random Gaussian matrix $\Omegam \in \mathbb{R}^{n \times (k+p)}$
\STATE Create $\Ym = (\Am\Am^T)^q \Am \Omegam$ by alternative multiplication with $\Am$ and $\Am^T$ \label{step:power} 
\STATE Compute $\Ym = \Vm\Rm$ using the QR-decomposition 
\STATE Form $\Bm = \Vm^T\Am$ and compute SVD $\Bm = \hat{\Pm}\Sigmam\Qm^T$
\STATE Set $\Pm = \Vm\hat{\Pm}$
\RETURN Approximate SVD $\Am \approx \Pm\Sigmam\Qm^T$ 
\end{algorithmic}\label{alg:randomSVD}
\end{algorithm}
This then gives us the primary ingredients we require for online matrix completion. At each iteration of soft impute we use the randomised SVD of Algorithm \ref{alg:randomSVD} to compute an approximate SVD. For a sequence of matrices $\Xm_1, \ldots, \Xm_T$ with corresponding nonzero indices $\omega_1, \ldots, \omega_T$ let the $i$th solution with regularisation parameter $\lambda$ be written $\Zm_\lambda^{[i]}$. To compute a solution for $\Xm_{i+1}$ we use the decomposition computed for $\Zm_\lambda^{[i]}$ as a seed for the first randomised SVD of soft impute. If $\Xm_i$ and $\Xm_{(i+1)}$ have different sizes then we can adjust the initial solution size accordingly, padding with zeros if $\Xm_{i+1}$ is larger than $\Xm_{i}$. 

There are two important advantages of this algorithm over the use of traditional Krylov subspace methods in conjunction with soft impute: the first is that the above algorithm can be effective even when there is not a jump in the spectrum in the incomplete matrices. Secondly, one can trivially compute $\Am \Omegam$ in parallel, which is the most expensive step. Hence, one can make full use of modern multi-core CPUs. 

\subsection{SVD of a Perturbed Matrix}\label{sec:svdUpdates}

It turns out that we can improve the efficiency of the online matrix completion approach further still by studying perturbations of the SVD. The problem of updating an SVD given a change is a rather important one as many problems such as clustering, denoising, and dimensionality reduction can be solved in part using the SVD.  A particular issue in recommendation is that a new user who has rated few items would not get accurate recommendations, known as the ``cold start'' problem \cite{schein2002methods}. In the matrix completion context, a special case of this problem has been considered in \cite{richard2010link} which studies the online learning of symmetric adjacency matrices.  Furthermore, an online approach for matrix completion albeit without trace norm regularisation is presented in \cite{balzano2010online} which uses gradient descent along the lines of Grassmannian and an incremental approach to compute solutions as columns are added. 

Consider the partial SVD given by $\Am_k = \Pm_k \Sigmam_k \Qm^T_k$, where $\Pm_k$, $\Qm_k$ have as columns the left and right singular vectors, and $\Sigmam_k$ has diagonal entries corresponding to the $k$ largest singular values. It is known that $\Am_k$ is the best $k$-rank approximation of $\Am$ using the Frobenius norm error. The change we are interested in can be encapsulated in a general sense as $\hat{\Am} = \Am + \Um$ in which $\Um \in \mathbb{R}^{m \times n}$. Note that one may also be interested in the addition of rows and columns to $\Am$. In this case, it is trivial to phrase these changes in terms of that given above by noting: $ \left[ \begin{array}{c} 
\Am \\
\textbf{0}  \\ 
\end{array} \right] =  \left[ \begin{array}{c} 
\Pm \\
\textbf{0}  \\ 
\end{array} \right] \Sigmam \Qm^T 
$ and $\left[ \begin{array}{c c} 
\Am & \textbf{0} 
\end{array} \right] =  \Pm 
 \Sigmam  \left[ \begin{array}{c c} 
\Qm & \textbf{0} 
\end{array} \right]^T,
$
and then adding an update matrix. Thus we will focus on finding the rank-$k$ approximation of $\hat{\Am}$, $\hat{\Am}_k$, given the first $k$ singular values and vectors of $\Am$ and the perturbation $\Um$. 
 
There is a range of work which focuses on the above problem when $\Um$ is low rank. One such method \cite{zha1999updating} uses the SVD of $\Am$ to approximate the SVD of $\Am + \Bm\Cm^T$, with  $\Bm \in \mathbb{R}^{m \times p}$ and $\Cm \in \mathbb{R}^{n \times p}$, without recomputing the SVD of the new matrix.  The total complexity is $\mathcal{O}((m+n)(pk + p^2) + k^3)$ where $k$ is the rank of $\Um$, however one requires first the decomposition of $\Um$ into $\Bm\Cm^T$. An improvement in complexity based on similar principals is provided in \cite{brand2006fast} which costs $\mathcal{O}(mnk)$ for $k \leq \sqrt{\min(m, n)}$. Unfortunately in our case the update is small in terms of the number of nonzero elements but typically has a large rank. 

To address this problem we present a simple randomised method for updating the SVD of $\Am$ given a sparse (but not necessarily low rank) update $\Um$. Unlike the updating methods mentioned above, we leverage our knowledge of the SVD algorithm to improve the approximation of the perturbed matrix $\hat{\Am}$. We use Algorithm \ref{alg:randomSVD} with $\Omegam = [\Qm_k \; \hat{\Omegam}]$ in which $\hat{\Omegam} \in \mathbb{R}^{n \times p}$ is a random Gaussian matrix. The idea is that one already has a good approximation of the first $k$ right singular vectors of $\Am$, which are not changed significantly by adding $\Um$ and we improve the projection matrix by adding $p$ random projections, where $p$ is small. Notice also the following: 

\begin{displaymath} 
 \hat{\Am}\Qm_k = \Pm_k\Sigmam_k + \Um\Qm_k,  
\end{displaymath}
and hence we need only compute $\Um\Qm_k$ and add it to precomputed $\Pm_k\Sigmam_k$. One need not use any power iteration (i.e. $q = 0$) and hence the complexity of the new step \ref{step:power} is $\mathcal{O}((p+nk) |\alpha| + p|\omega|)$ where $\alpha$ is the set of nonzero entries in $\Um$. Of note is that this step requires only a single scan of $\hat{\Am}$ versus the $2q + 1$ required in Algorithm \ref{alg:randomSVD} and yet a highly accurate solution is obtained, as we shall later see. 

\subsection{Analysis}

Here we study the online matrix completion method by looking at the error introduced by using the approximate SVD. Consider again Algorithm \ref{alg:randomSVD} whose error is bounded by the following theorem. 

\begin{theorem}{\cite{halko2011finding}} 
Define $\Am \in \mathbb{R}^{m \times n}$. Select an exponent $q$, a target number of singular vectors $k$ with $2 \leq k \leq 0.5\min(m, n)$ and let $p=k$. Algorithm \ref{alg:randomSVD} returns a rank $2k$ factorisation $\Pm\Sigmam\Qm^T$ such that 
\begin{displaymath}
 \mathbb{E}\|\Am - \Pm\Sigmam\Qm^T\|_2 \leq \left[1+ 4\sqrt{\frac{2\min(m,n)}{k-1}} \right]^{1/(2q+1)} \sigma_{k+1},  
\end{displaymath}
where $\mathbb{E}$ is the expectation with respect to the random projection matrix, $\|\cdot\|_2$ is the spectral norm and $\sigma_{k+1}$ is the $(k+1)$th largest singular value of $\Am$. \label{thm:rsvd}
\end{theorem}

Of note from this theorem is that the error decreases the term in square brackets exponentially fast as $q$ increases. Furthermore, the expectation is shown to be almost always close to the typical outcome due to measure concentration effects. We can now study the error introduced by the randomised SVD to each iteration of Algorithm \ref{alg:softImpute}. 

\begin{theorem} 
Define $f_\lambda(\Zm) = \frac{1}{2} \|P_\omega(\Xm) - P_\omega(\Zm)\|_F^2 + \lambda \|\Zm\|_*$ and let $\Zm = S_{\lambda}(\Ym)$ for some matrix $\Ym$. Furthermore, denote by $\hat{S}_\lambda$ the soft thresholding operator using the SVD as computed using Algorithm \ref{alg:randomSVD} with $p=k$ and let $\hat{\Zm} = \hat{S}_{\lambda}(\Ym)$. Then the following bound holds: 
\begin{eqnarray*} 
&& \mathbb{E}|f_\lambda(\Zm) - f_\lambda(\hat{\Zm})| \leq  \lambda k (1 + \theta)\sigma_{k+1} + \\
&&  \quad \frac{1}{2} \|\sigmav_{>k}\|_2^2 + k \left(1 + \frac{k}{k-1}\right)^{\frac{1}{2q+1}} \|\sigmav_{>k}\|_{(2q+1)}, 
\end{eqnarray*}
where $k$ is the rank of the partial SVDs, $\sigma_1 \geq \sigma_2 \geq \cdots \geq \sigma_r$ are the singular values of $\Ym$, $\sigmav_{>k} = [\sigma_{k+1} \; \cdots \; \sigma_r]^T$, and $\theta = \left[1+ 4\sqrt{\frac{2\min(m,n)}{k-1}} \right]^{1/(2q+1)}$. \label{thm:fBound}
\end{theorem}

The proof of this theorem is deferred to the appendix. Notice that the first two terms on the right side of the inequality come from Theorem \ref{thm:rsvd} and constitute the error in the approximation of the trace norm of $\Zm$. The final terms represent the bound of the Frobenius norm difference between the approximate and real SVDs, the last of which will tend towards $k \sigma_{k+1}$ as $q$ increases. Naturally, the bound is favourable when the singular values of the residual matrix after $k$ are small.  An interesting consequence of the bound in conjunction with soft impute is that successive matrices $\Ym$ (or equivalently $\Zm^{(j)}$ in Algorithm \ref{alg:softImpute}) have decreasing singular values towards the end of the spectrum and hence the error between $f_\lambda(\Zm)$ and $f_\lambda(\hat{\Zm})$ decreases as one iterates. This helps to explain the good convergence of the online algorithm. We later examine the error in the randomised SVD in practice.  

\section{Computational Results}\label{sec:results}

In this section we highlight the efficacy of online matrix completion approach on one toy and several real datasets. We compare the online algorithm in conjunction with the randomised SVD approach at each iteration, the SVD updating method of Section \ref{sec:svdUpdates}, and PROPACK, denoted \texttt{RVSD}, \texttt{RSVD+} and \texttt{PROPACK} respectively, fixing $\epsilon = 10^{-3}$. Note that comparisons have already been made in \cite{mazumder2010spectral} with MMMF and SVT on several datasets with competitive results in the offline case. All experimental code is implemented in Python with critical sections implemented in C++. For \texttt{RVSD} and \texttt{RSVD+} we parallelise the multiplication of a sparse matrix by a projection matrix. We attempted the same parallelisation strategy for \texttt{PROPACK} however the algorithm works using matrix-vector multiplications for which the overhead of parallelisation exceeded the computational gains made. Timings are recorded on a 24 core Intel Xeon X5670 CPU with 192GB RAM. 
 
\subsection{Synthetic Datasets}

To begin with we consider a simple dataset generated using the following process. There are 20 matrices in the sequence, of which the first 10 are in $\mathbb{R}^{5000 \times 1000}$, and from the 10th to the 20th matrix sizes are increased uniformly in both dimensions to $\mathbb{R}^{10000 \times 1500}$. The fully observed matrix is of rank $r=50$ with a decomposition of the form $\Pm\Sigmam\Qm^T$ such that $\Pm, \Qm$ are random orthogonal matrices and the singular values $\Sigmam_{ii}$s are uniformly randomly selected, with $\Sigmam_{ij} = 0$, $i \neq j$. For the first 10 matrices in the sequence, elements are observed initially with probability $0.03$ and this increases in equal steps to $0.10$. Individual elements are normalised to have a standard deviation 1 and we add a noise term $\mathcal{N}(0, 0.01)$. For each matrix in the sequence we split the observations into a training and test set of approximately the same sizes. As a preprocessing step, the nonzero elements of each row of the training matrices are centered according to the mean value of the corresponding row, and the equivalent transformation is applied to the test matrices. Note that we studied the rank of these matrices and found them to be nearly full rank despite being generated by a low rank decomposition. Furthermore the differences between successive matrices was high rank, making the low rank SVD update strategies of Section \ref{sec:svdUpdates} impractical. 

Before looking at error rates, we first studied the subspaces generated by soft impute in conjunction with \texttt{PROPACK} with the aim of observing how they change as the algorithm iterates on the training matrices. The same dataset of 20 matrices is used however we reset soft impute at each matrix so that the initial solution $\Zm^{(1)} = \textbf{0}$. We use a value of $k = 50$ for the partial SVD for each iteration of the algorithm which corresponds to the rank of the underlying matrices. Instead of directly using the $\lambda$ parameter, which is sensitive to variations in the size and number of observed entries in a matrix, we use $\rho = \lambda/\sigma_1$ where $\sigma_1$ is the largest singular value of $\Xm_i$ and $\rho = 0.5$ in this case. As well as recording $\gamma$  we also compute at the $i$th iteration

\begin{displaymath} 
\theta_{\Pm_i} = \frac{\|(\Imat - \Pm_{(i-1)}\Pm_{(i-1)}^T)\Pm_i\|_F^2}{\|\Pm_i\|_F^2} = \frac{k - {\|\Pm_i^T\Pm_{i-1}\|_F^2}}{k}, 
\end{displaymath}
\noindent
where $k$ is the dimensionality of $\Pm_i$ which is the matrix of left singular vectors of $\Zm^{(i)}$. The corresponding measure for the right singular vectors of $\Zm^{(i)}$ $\Qm_i$, is denoted $\theta_{\Qm_i}$. We additionally compute the change in the thresholded singular values in a similar fashion, $\phi_{\sigmav_i} = \|\sigmav_{i} - \sigmav_{i-1}\|^2/\|\sigmav_i\|^2$, where $\sigmav_{i}$ is the soft thresholded singular values of the $i$th matrix. These measures are computed over all 20 training matrices and averaged. 
\begin{figure}[ht]
 \centering
 \includegraphics[width=0.45 \linewidth]{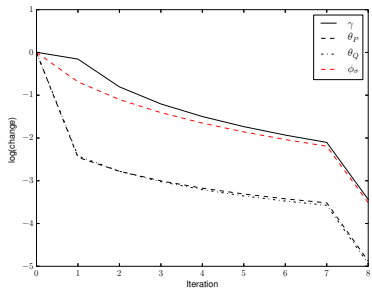}
 \caption{Mean differences in successive left and right subspaces of $\Zm$ and error measure $\gamma$ for soft impute. } \label{fig:syntheticSubspace} 
\end{figure}
It is clear from Figure \ref{fig:syntheticSubspace} that the left and right subspaces of $\Zm$ rapidly decrease after the first iteration to approximately $3.7 \times 10^{-3}$ and continue to fall. The error $\gamma$ decreases at a slower rate and we see that most of the change occurs with the soft thresholded singular values of $\Zm$. The important point to take note of is that the slowly varying subspaces of $\Zm$ play into the updating of the randomised SVD presented in Section \ref{sec:svdUpdates}.

Next we evaluate the generalisation error of the matrix completion approaches by recording the root mean squared error (RMSE) on the observed entries,
\begin{displaymath} 
\mbox{RMSE}(\Xm, \hat{\Xm}) = \sqrt{\frac{1}{|\omega|} \sum_{(i,j) \in \omega} (\Xm_{ij} - \hat{\Xm}_{ij})^2}, 
\end{displaymath} 
\noindent
where $\hat{\Xm}_{ij}$ are the predicted elements. Furthermore, we postprocess the singular values as suggested in \cite{mazumder2010spectral} so as to minimise $\|P_{\omega}(\Xm) - \sum_i \sigma_i P_\omega(\pv_i\qv_i^T)\|_F^2$, for $\sigma_i \geq 0$, $i=0,\ldots,r$, using a maximum of $10^6$ nonzero elements. The RMSEs are recorded on the training and test observations. The experiment is repeated using \texttt{RVSD}, \texttt{RSVD+} and \texttt{PROPACK} in conjunction with soft impute. With \texttt{RVSD} and \texttt{RSVD+} we explore different values of $p$ and $q$, in particular the following pairs of values are used: $\{(10, 2), (50, 2), (10, 5)\}$. For the $(10, 5)$ case we also compute using cold restarts (i.e $\Zm^{[i]} = \textbf{0}$, $\forall i$) to contrast it to the use of previous solutions. 
\begin{figure}[ht]
 \centering
 \subfigure {\includegraphics[width=0.45 \linewidth]{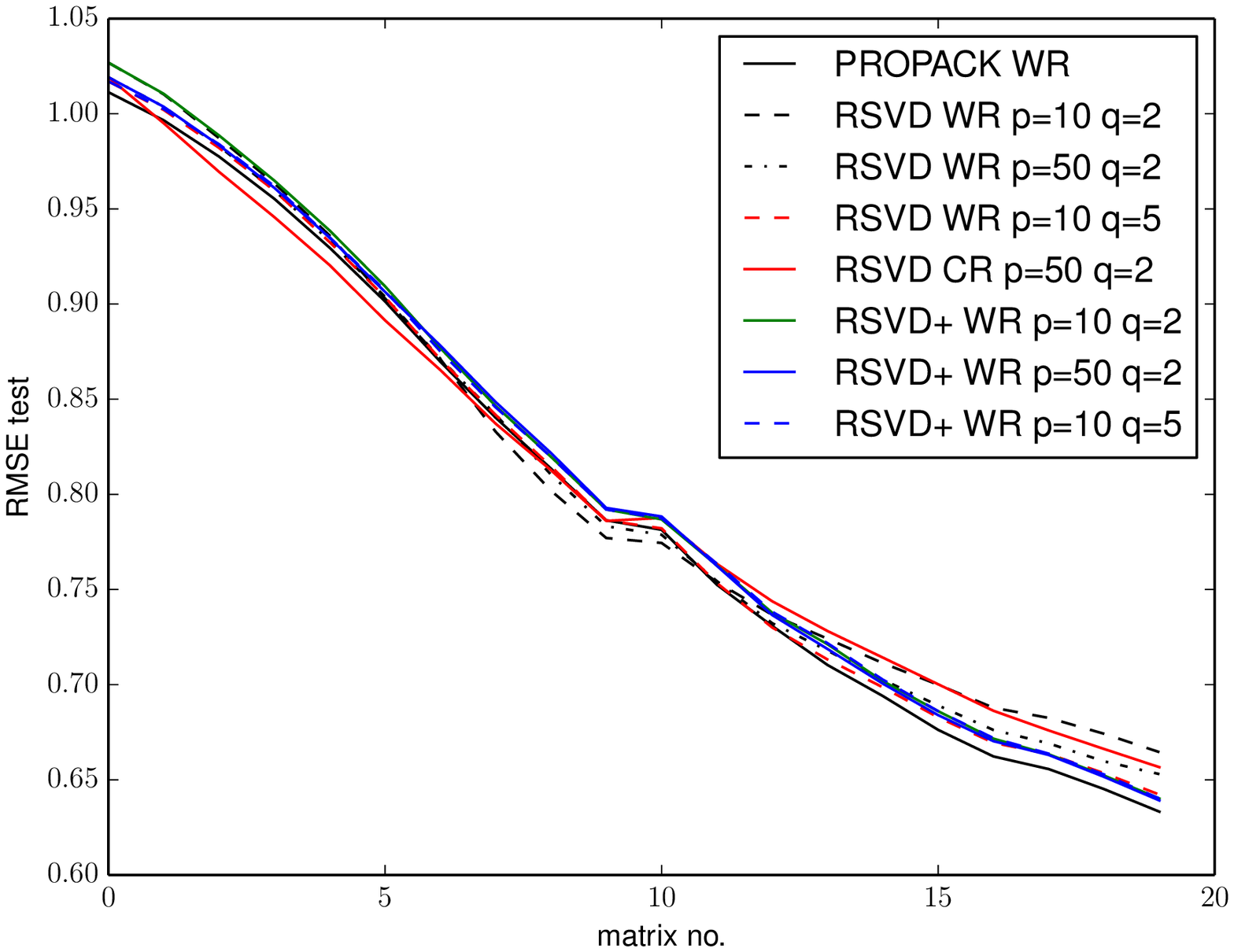}}
\subfigure{ \includegraphics[width=0.45 \linewidth]{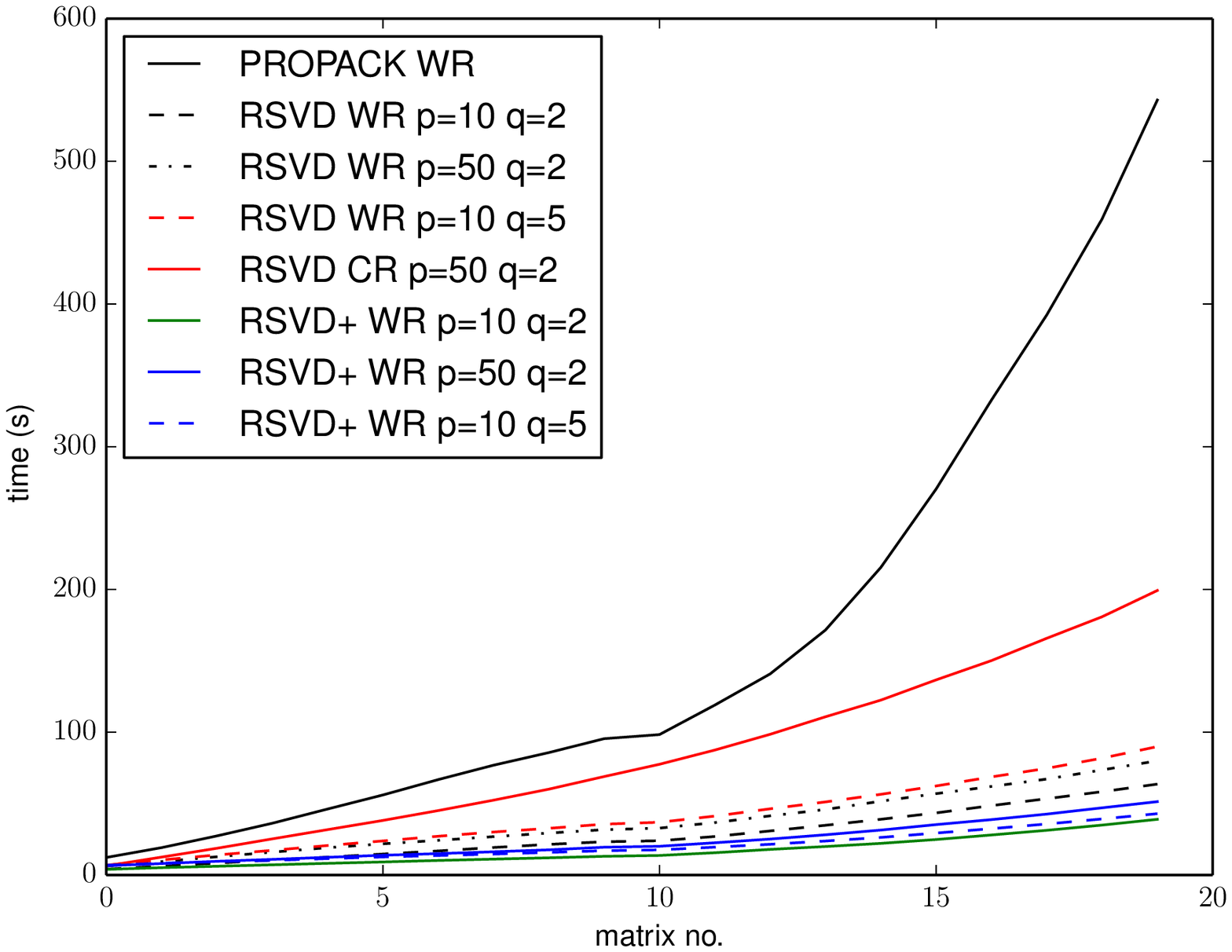} } 
 \caption{Errors and timings on the synthetic dataset, using Warm Restarts (WR) and Cold Restarts (CR).} \label{fig:synthetic} 
\end{figure}
The errors of the test observations are shown Figure \ref{fig:synthetic}. The randomised SVD methods are very similar although slightly worse than \texttt{PROPACK} for the most of the matrices in the sequence. Interestingly \texttt{RSVD+} matches or improves over \texttt{RSVD} particular towards the end of the sequence. The effect of cold restarts on the error is mixed: for the initial 10 matrices it provides an improvement in error however for the final 10 it seems to be a disadvantage. The key point however, is that the timing are considerably better for the randomised methods with warm restarts compared to both \texttt{PROPACK} and \texttt{RSVD} with cold restarts. With \texttt{RSVD+} $p=10, q=2$ for example, the total time for completing the entire sequence of matrices is 38.9s versus 543s for \texttt{PROPACK} and yet the error is only 0.007 larger for the final matrix.  

\subsection{Real Datasets}

Next we use the online matrix completion algorithms in conjunction with three real datasets. The Netflix dataset contains ratings on a scale of 1 to 5 augmented with the date the rating was made which allows us to explore recommendation accuracy with time. We start with the set of ratings made by the end of 2003 and compute predictions for rating matrices at 30 day intervals, a total of 26 matrices. MovieLens 10M has ratings from 0 to 5 in steps of 0.5. We start at the rating matrix at the end of 2004, incrementing in 30 day intervals until March 2008 giving a total of 40 matrices. Finally, we use the Flixster movie dataset after processing it so that we keep only movies and users with 10 ratings or more, which are given on a scale of 1 to 5 in steps of 0.5. Ratings are used from January 2007 to November 2009 in 30 day increments, resulting in 21 matrices. Table \ref{tab:datasets} gives some information about these datasets. For each matrix we use a training/test split of 0.8/0.2, preprocessed by centering the rows as described for the synthetic data above. 
\begin{table}
\begin{center}
\begin{tabular}{l | l l l l} 
 \hline
& m & n & $|\omega_1|$ & $|\omega_T|$  \\
 \hline 
Flixster & 50,130 & 18,369 & 6,376,264 & 7,825,955  \\
ML & 71,567 & 10,681 & 6,569,292 & 10,000,054 \\ 
Netflix & 480,189 & 17,770 & 17,023,860 & 100,480,507 \\ 
 \hline 
\end{tabular} 
\end{center}
\caption{Information about the real datasets.} \label{tab:datasets}
\end{table}
For the initial matrix in the sequence we perform model selection on the training set in order to set the parameters of the matrix completion methods. This is performed on a sample of at most $5 \times 10^6$ randomly sampled elements the training matrix using 5-fold cross validation. We select $\rho$ from $\{0.05, 0.1, \ldots, 0.4\}$ and $k$ is chosen from $\{8, 16, 32, 64, 128\}$. We postprocess the singular values as described above and set $p=50, q=2$, $p=50, q=3$ and $p=10, q=3$ for \texttt{RSVD} and \texttt{RSVD+}. As before, we train on the training observations and record the RMSE on the test observations. 
\begin{figure}[ht]
 \centering
 \subfigure {\includegraphics[width=0.45 \linewidth]{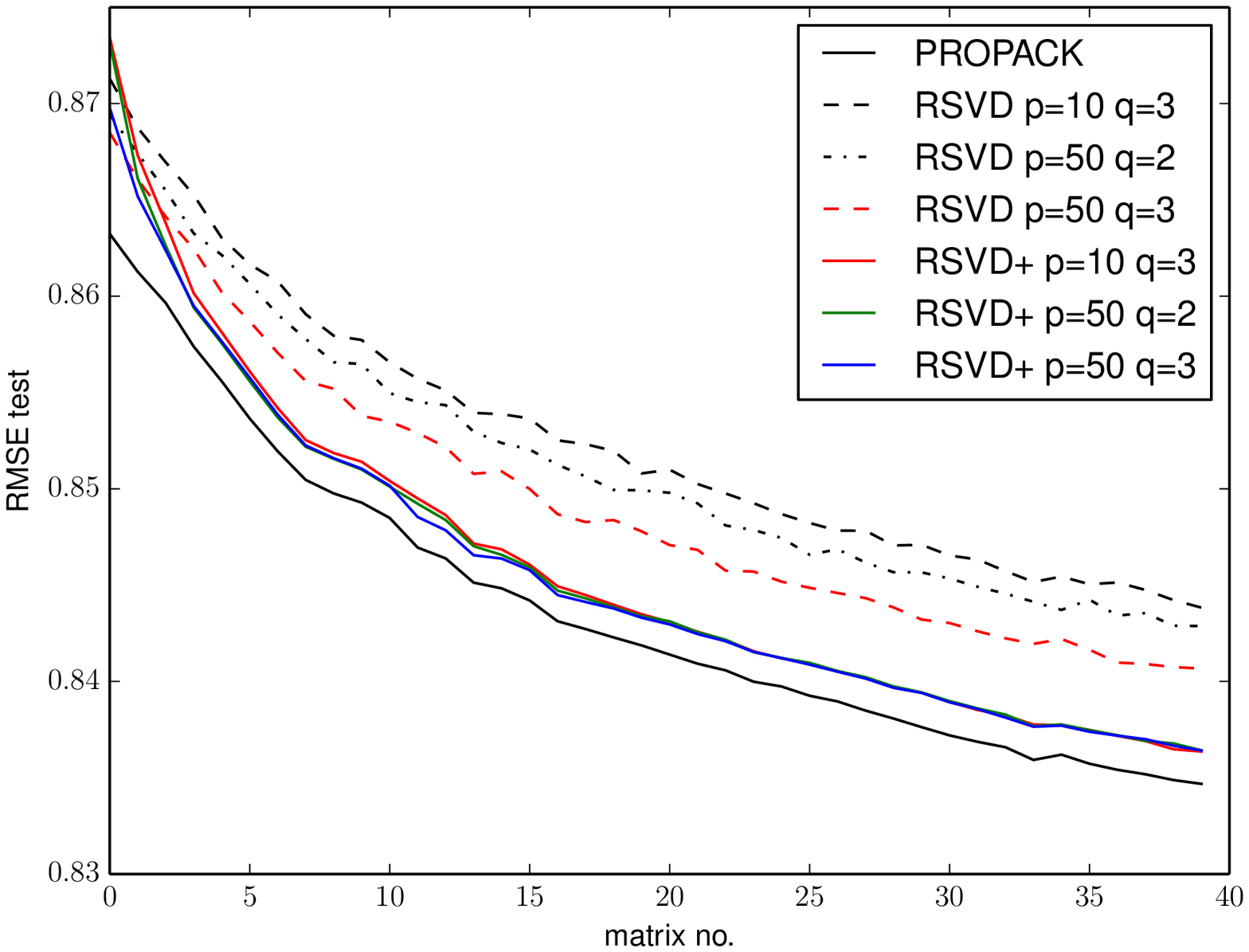}\label{fig:movielensError} }
 \subfigure{\includegraphics[width=0.45 \linewidth]{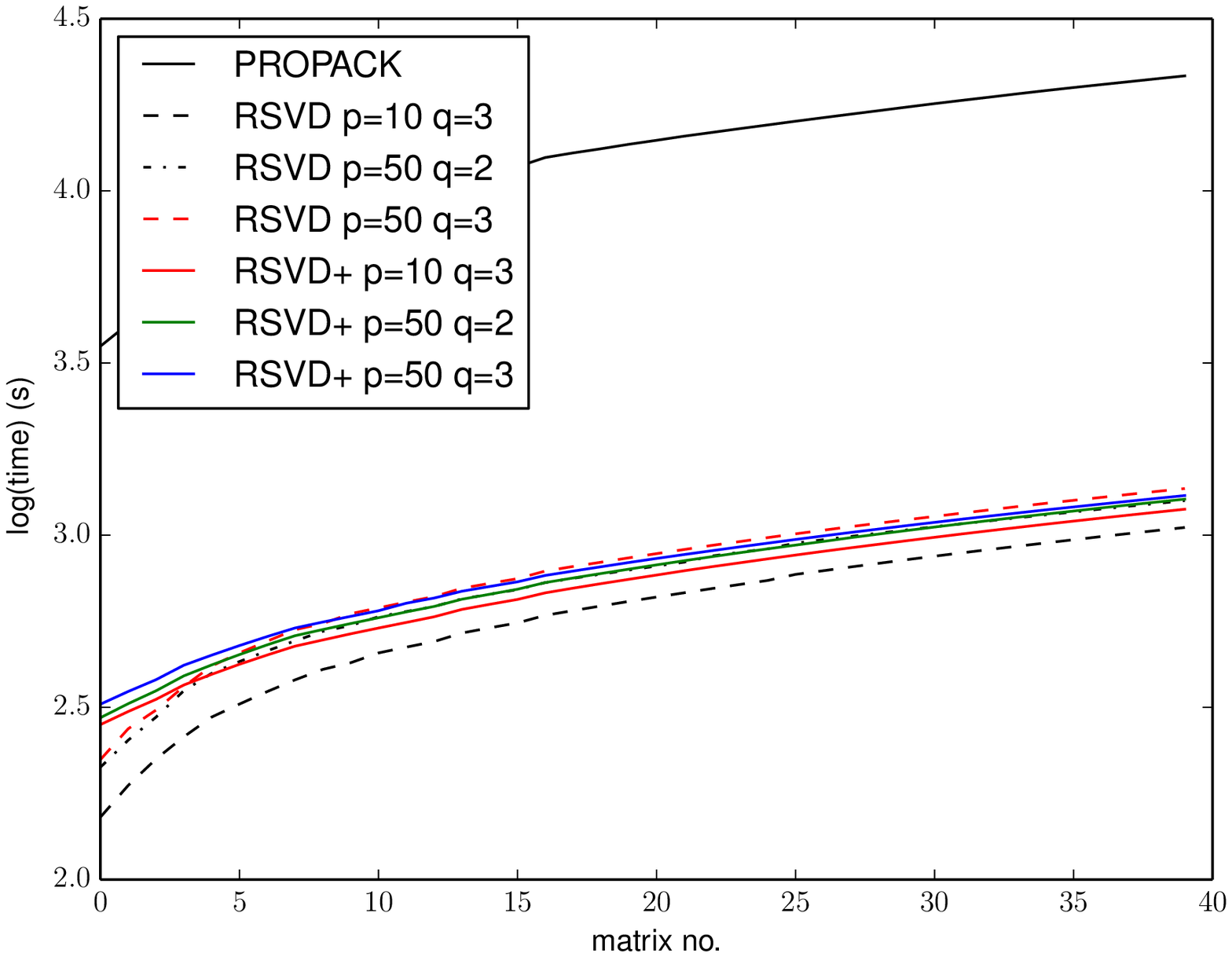}\label{fig:movielensTime} } 
 \caption{Errors and timings on the MovieLens dataset} \label{fig:movielens} 
\end{figure}

\begin{figure}[ht]
 \centering
 \subfigure {\includegraphics[width=0.45 \linewidth]{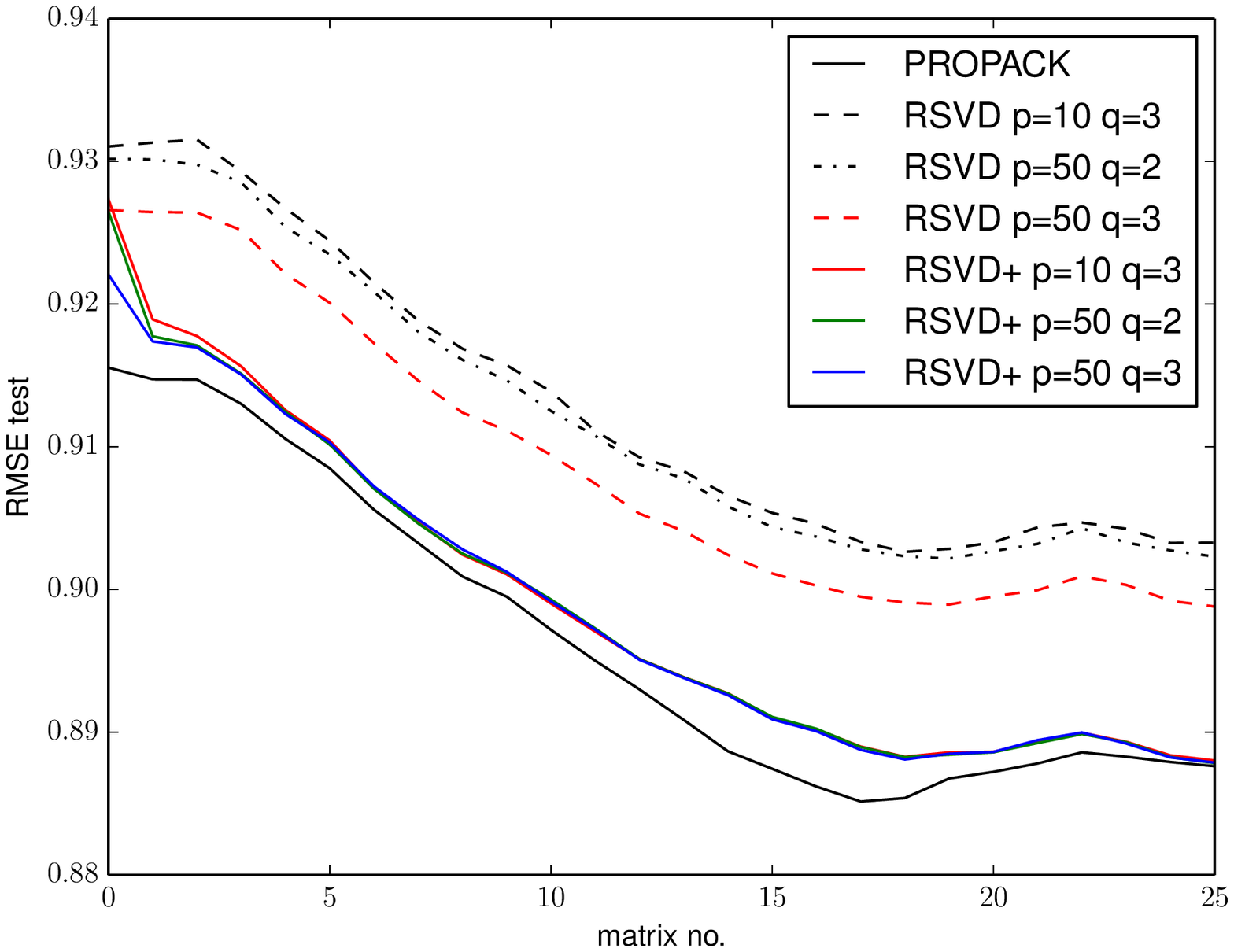}}
 \subfigure{\includegraphics[width=0.45 \linewidth]{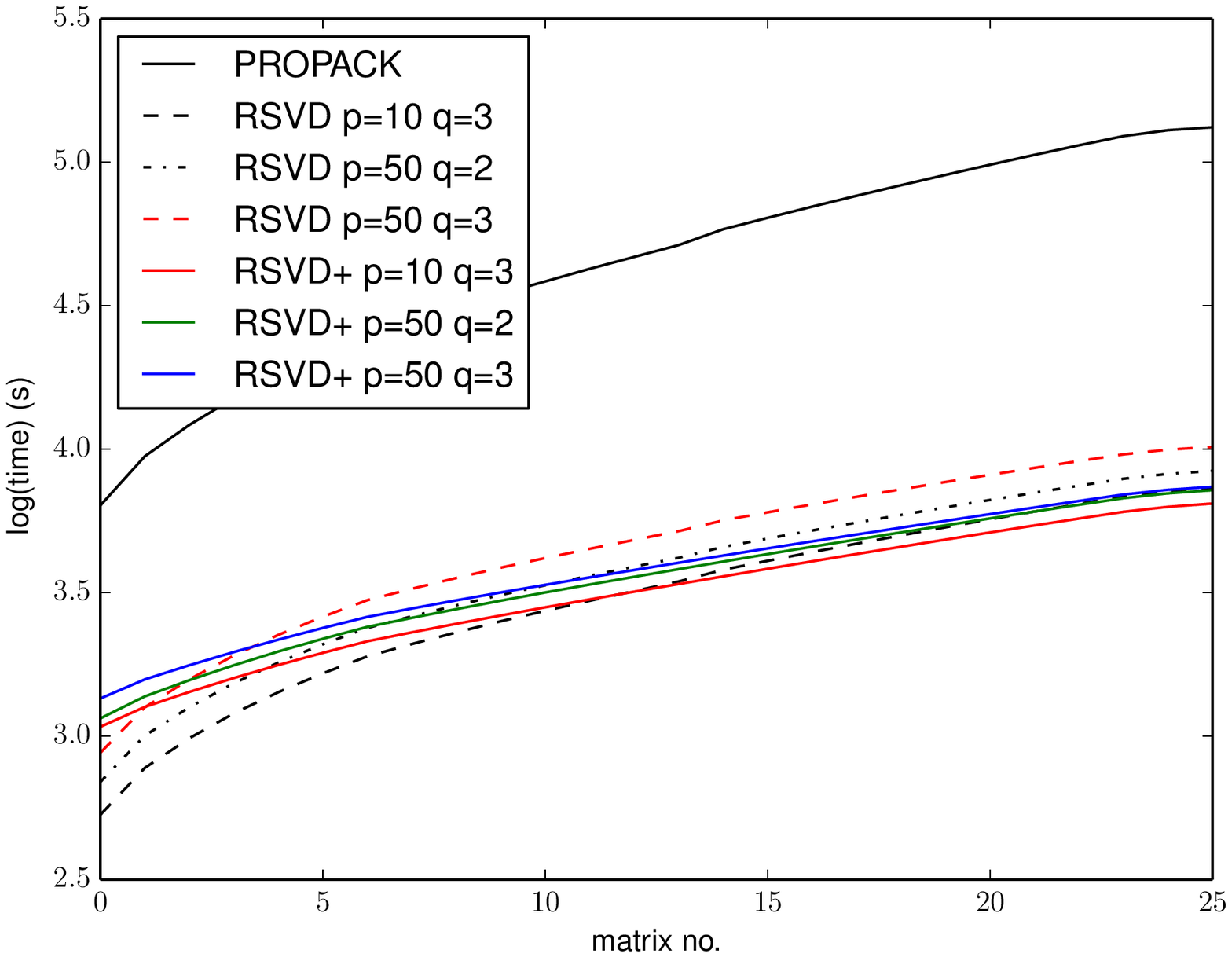} } 
 \caption{Errors and timings on the Netflix dataset} \label{fig:netflix} 
\end{figure}

\begin{figure}[ht]
 \centering
 \subfigure {\includegraphics[width=0.45 \linewidth]{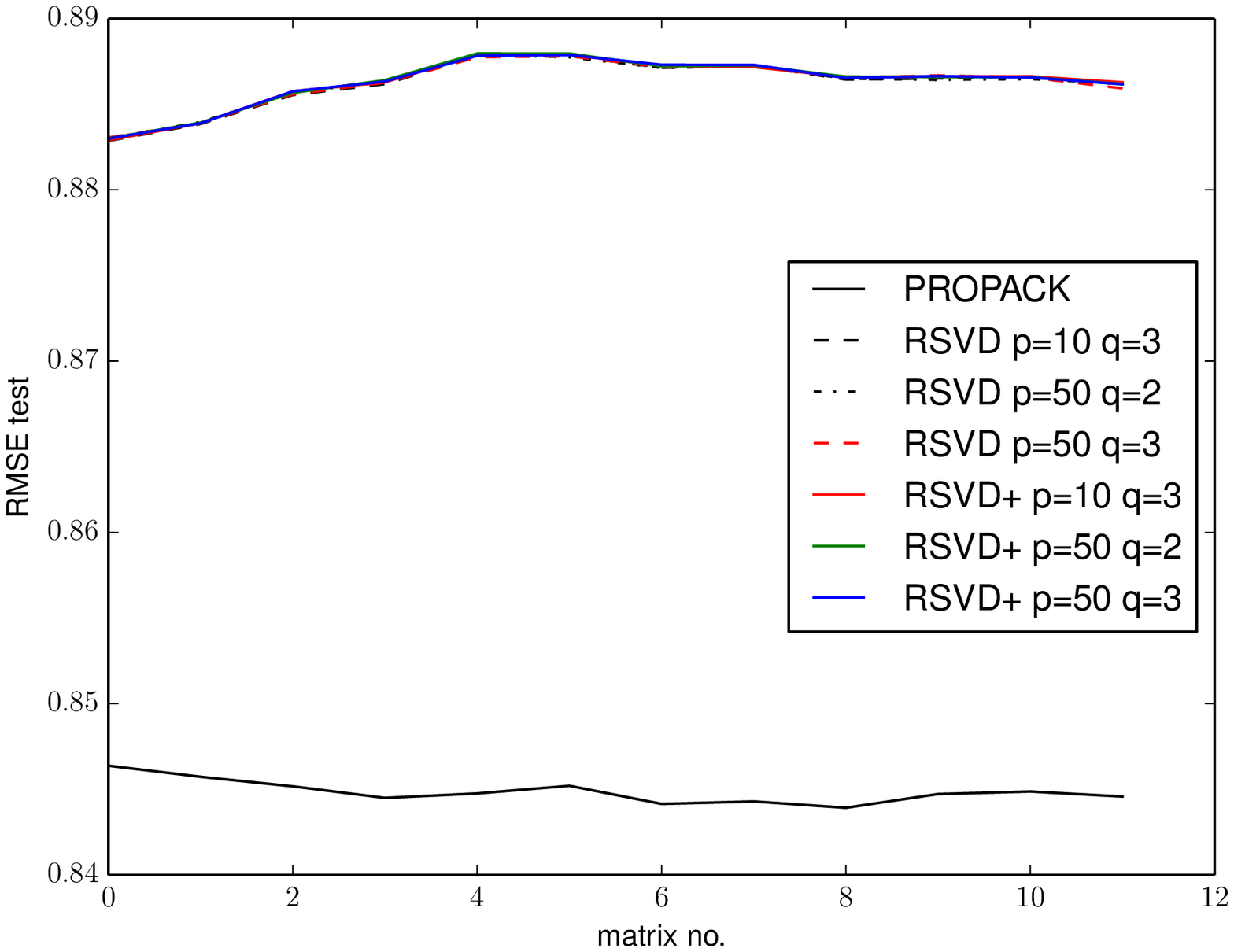}}
 \subfigure {\includegraphics[width=0.45 \linewidth]{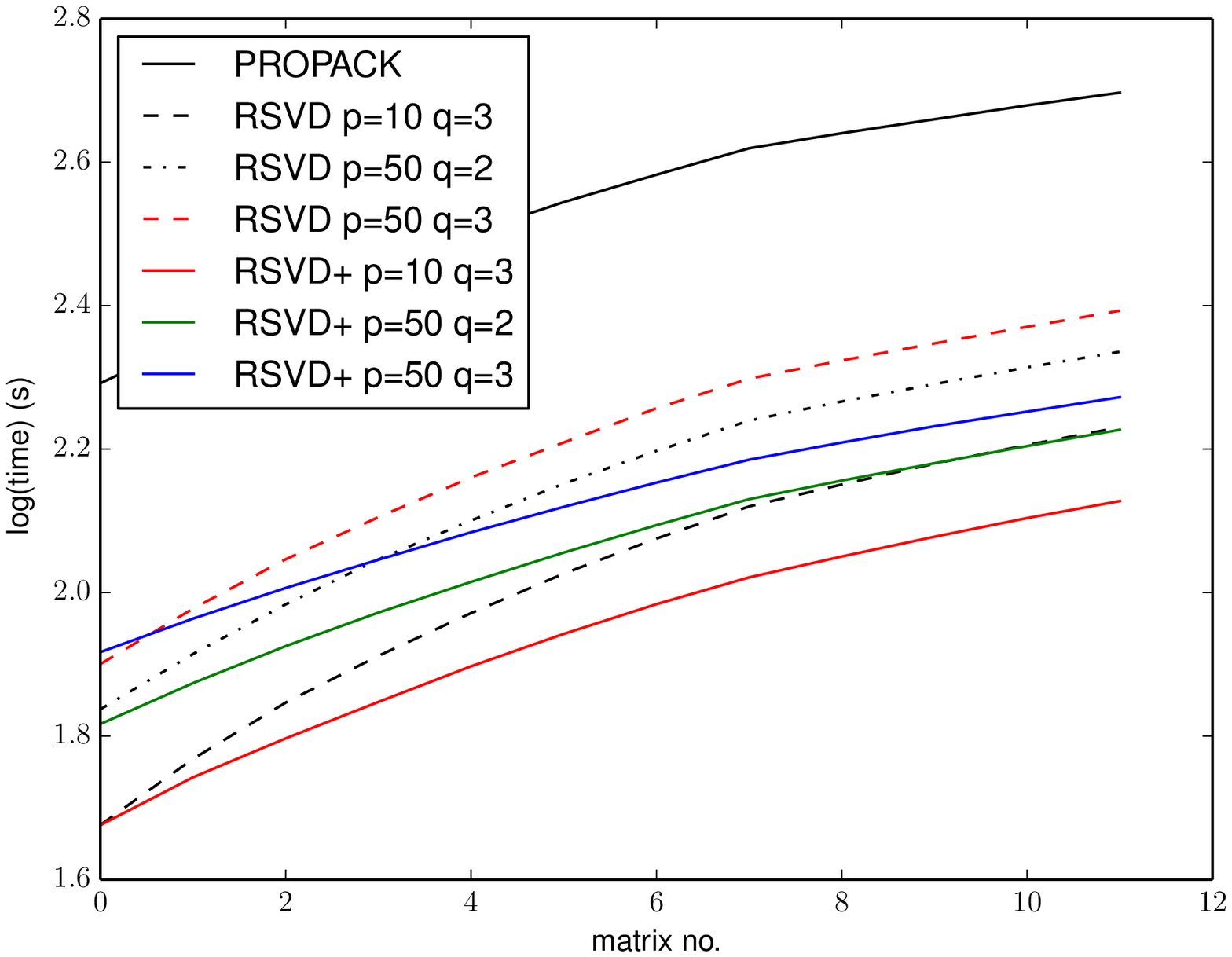}}
 \caption{Errors and timings on the Flixster dataset} \label{fig:flixster} 
\end{figure}
Figures \ref{fig:movielens} show the resulting errors and timings for the MovieLens dataset. The prediction errors generally improve over time as one has access to more entries of the incomplete matrices. We see that \texttt{PROPACK} takes considerably longer to complete the sequence of matrices, however it does have a slight advantage in error. \texttt{PROPACK} requires 21,513 seconds to complete all the matrices  whereas \texttt{RSVD+} $p=10$, $q=3$ took just 1189 seconds, both for rank-$128$ solutions. The differences in timings between the \texttt{RSVD+} methods is small since power iteration is only used for the initial solutions.  \texttt{RSVD} favours a rank-$64$ decomposition due to errors in the error grid, and this explains why the corresponding RMSEs are worse than the other methods and timings are better with $p=10$, $q=3$ relative to \texttt{RSVD+} for example. The final difference in RMSE between \texttt{RSVD+} $p=10$ $q = 3$ and \texttt{PROPACK} is 0.0016. 

On Netflix, Figure \ref{fig:netflix}, we see that \texttt{PROPACK} requires 132,354 seconds to complete all matrices versus 6463 seconds for \texttt{RSVD+} $p=10$, $q=3$, an improvement factor of approximately 20. The difference in RMSE at the final matrix between these methods is negligible at 0.0004. Note that the \texttt{RSVD} methods choose $k=64$ versus $k=128$ for \texttt{RSVD+}. If we compare \texttt{RSVD} and \texttt{RSVD+} for $p=50$, $q=3$, the latter improves the error of the former at the final matrix by 0.01 and timings are 10187 and 7394 respectively. As with MovieLens we see that under difference values of $p$ and $q$ the \texttt{RSVD+} results converge to similar errors. If the spaces of the top $k$ singular vectors of the input matrices change only slight then it follows that repeated sampling in the manner used in \texttt{RSVD+} will produce increasingly accurate results, a key strength of the method in this iterative context.  

Finally, we come to Flixster in Figure \ref{fig:flixster} and on this dataset the methods all chose $k=8$ $\rho = 0.05$ during model selection and the randomised methods have very similar errors. Of note from the plots is that the matrices do not become easier to complete over time, and \texttt{PROPACK} improves the error of the randomised methods by approximately 0.04. We believe that this dataset is particularly sensitive to round-off errors in the random SVD procedure.  The timings shown in Figure \ref{fig:flixster} show that there is a clear advantage in the \texttt{RSVD+} methods over \texttt{RSVD}, for example observe the respective cumulative timings for $p=50$, $q=3$, 187s and 247s with negligible difference in error. 

\section{Conclusions}\label{sec:conclusion} 

We addressed a critical issue in practical applications of matrix completion, namely online learning, based on soft impute which uses a trace norm penalty. The principal bottleneck of this algorithm is the computation of the SVD using PROPACK, which is not readily parallelisable and we motivated the randomised SVD for efficiently evaluating the SVD. Additionally, we showed how matrix completion can be conducted in an online setting by using previous solutions and a method to update the SVD under a potentially high-rank change. The resulting algorithm is simple to implement, and easily parallelisable for effective use of modern multi-core computer architectures. The expectation of the error of the algorithm in terms of the objective function is bounded theoretically, and empirical evidence is provided on its efficacy. In particular, on the large MovieLens and Netflix datasets, RSVD+ significantly reduces computational time upon PROPACK for a small penalty in error, and improves the efficiency of the randomised SVD. 

Our novel SVD updating method opens up work in other algorithms which require the computation of SVDs or eigen-decompositions under high rank perturbations. We plan to study theoretically in more detail the error and convergence properties of our online matrix completion approach. 

\appendix

\section{Proof of Theorem \ref{thm:fBound}}

This appendix section details the proof of the main theorem in the article. We start with a result which is analogous to Theorem \ref{thm:rsvd} and bounds the Frobenius norm error of a random projection. 

\begin{theorem}{\cite{halko2011finding}}  
Suppose that $\Am \in \mathbb{R}^{m \times n}$ with singular values $\sigma_1 \geq \sigma_2 \geq \cdots \geq \sigma_r$. Choose a target rank $k$ and an oversampling parameter $p \geq 2$ where $k + p \leq \min(m, n)$. Draw an $n \times (k + p)$ standard Gaussian matrix $\Omegam$, and construct $\Ym = \Am \Omegam$. Then the expected approximation error is bounded by  

\begin{displaymath} 
 \mathbb{E}\|(\Imat - \Pm_\Ym)\Am \|_F \leq \left( 1 + \frac{k}{p-1} \right)^{1/2} \left( \sum_{j>k}\sigma_j^2\right)^{1/2},
\end{displaymath}
where $\Pm_\Ym$ is the projection matrix constructed from the orthogonalisation of $\Ym$. \label{thm:froError}
\end{theorem}

Note that this error bound is the same as that between $\Am$ and its randomised SVD, $\Pm\Sigmam\Qm^T$, since from Algorithm \ref{alg:randomSVD} $\Pm\Sigmam\Qm^T = \Vm\Vm^T\Am$. We now introduce another result which characterises the error using the power scheme 

\begin{theorem}{\cite{halko2011finding}}   
 Suppose that $\Am \in \mathbb{R}^{m \times n}$ and let $\Omegam$ be an $n \times \ell$ matrix. For some nonnegative integer $q$, let $\Bm = (\Am\Am^T)^q\Am$ and compute $\Ym = \Bm\Omegam$, then 
\begin{displaymath} 
 \|(\Imat - \Pm_\Ym)\Am \|_2 \leq  \|(\Imat - \Pm_\Ym)\Bm \|^{1/(2q+1)}_2,
\end{displaymath}
and it follows that 
\begin{displaymath} 
 \|(\Imat - \Pm_\Ym)\Am \|_F \leq  \sqrt{\ell} \|(\Imat - \Pm_\Ym)\Bm \|^{1/(2q+1)}_2. 
\end{displaymath} \label{thm:power}
\end{theorem}

This gives us the necessary ingredients to derive a bound on the expected Frobenius norm error. 

\begin{theorem} 
Define rank-$r$ matrix $\Am \in \mathbb{R}^{m \times n}$, select an exponent $q$, a target number of singular vectors $k$ and a nonnegative oversampling parameter $p$. Let $\Bm = (\Am\Am^T)^q\Am$ and compute $\Ym = \Bm\Omegam$ in which $\Omega$ is an $n \times (k + p)$ standard Gaussian matrix. Then the expected approximation error is bounded by  
\begin{displaymath} 
  \mathbb{E}\|(\Imat - \Pm_\Ym)\Am \|_F^2 \leq (k+p)  \left(1 + \frac{k}{p-1} \right)^{\frac{1}{2q+1}} \| \sigmav_{>k}\|_{(2q+1)},  
\end{displaymath}
where $\sigmav_{>k}$ is a vector of singular values $\sigma_{k+1}, \ldots, \sigma_r$ and $\| \cdot \|_{p}$ is the $p$-norm of the input vector. \label{thm:froSVDError}
\end{theorem}
\begin{proof}
We begin by using H\"{o}lder's inequality to bound $\mathbb{E}\|(\Imat - \Pm_\Ym)\Am \|_F^2$ 
\begin{eqnarray*}
  &\leq&  \left(\mathbb{E} \|(\Imat - \Pm_\Ym)\Am  \|_F^{2q+1} \right)^{\frac{2}{2q+1}} \\ 
&\leq& \left( (k+p) ^{\frac{2q+1}{2}} \mathbb{E} \|(\Imat - \Pm_\Ym)\Bm  \|_2 \right)^{\frac{2}{2q+1}} \\ 
&\leq& (k+p) \left(  \left(1 + \frac{k}{p-1} \right)\sum_{j >k} \sigma^{(2q+1)}_j \right)^{\frac{1}{2q+1}},
\end{eqnarray*}
where the second step makes use of  Theorem \ref{thm:power}. The final line uses Theorem \ref{thm:froError} noting both that $\|\Xm\|_2 \leq \|\Xm\|_F$ for any matrix $\Xm$, and the singular values of $\Bm$ are given by $\sigma_1^{2q+1}, \ldots, \sigma_r^{2q+1}$.
\end{proof}

Before introducing the main theorem we present a lemma pertaining to the norm of thresholded SVDs. 

\begin{lemma}{\cite{mazumder2010spectral, ma2011fixed}}
The shrinkage operator $S_\lambda(\cdot)$ satisfies the following for any $\Wm_1$ and $\Wm_2$ with matching dimensions: 
\begin{displaymath} 
\|S_\lambda(\Wm_1) - S_\lambda(\Wm_2) \|^2_F \leq \|\Wm_1 - \Wm_2\|_F^2,
\end{displaymath}
which implies $S_\lambda(\Wm)$ is a continuous map in $\Wm$. \label{lem:normSoft}
\end{lemma}

We can now study the error introduced by the randomised SVD to each iteration of Algorithm \ref{alg:randomSVD}. 

\begin{theorem} 
Define $f_\lambda(\Zm) = \frac{1}{2} \|P_\omega(\Xm) - P_\omega(\Zm)\|_F^2 + \lambda \|\Zm\|_*$ and let $\Zm = S_{\lambda}(\Ym)$ for some matrix $\Ym$. Furthermore, denote by $\hat{S}_\lambda$ the soft thresholding operator using the SVD as computed using Algorithm \ref{alg:randomSVD} with $p=k$ and let $\hat{\Zm} = \hat{S}_{\lambda}(\Ym)$. Then the following bound holds: 
\begin{eqnarray*} 
&& \mathbb{E}|f_\lambda(\Zm) - f_\lambda(\hat{\Zm})| \leq  \lambda k (1 + \theta)\sigma_{k+1} + \\
&&  \quad \frac{1}{2} \|\sigmav_{>k}\|_2^2 + k \left(1 + \frac{k}{k-1}\right)^{\frac{1}{2q+1}} \|\sigmav_{>k}\|_{(2q+1)}, 
\end{eqnarray*}
where $k$ is the rank of the partial SVDs, $\sigma_1 \geq \sigma_2 \geq \cdots \geq \sigma_r$ are the singular values of $\Ym$, $\sigmav_{>k} = [\sigma_{k+1} \; \cdots \; \sigma_r]^T$, and $\theta = \left[1+ 4\sqrt{\frac{2\min(m,n)}{k-1}} \right]^{1/(2q+1)}$. 
\end{theorem}
\begin{proof} 
Define $T_{k}(\cdot)$ and $\hat{T}_{k}(\cdot)$ respectively as the $k$-rank SVD and randomised SVD of the input, then $\mathbb{E}\|T_{k}(\Ym) - \hat{T}_{k}(\Ym)\|_2$ is
\begin{eqnarray*} 
  &=& \mathbb{E}\|T_{k}(\Ym) - \Ym + \Ym - \hat{T}_{k}(\Ym)\|_2  \\
  &\leq& \mathbb{E}(\|T_{k}(\Ym) - \Ym\|_2 + \|\Ym - \hat{T}_{k}(\Ym)\|_2) \\
  &\leq& (1+\theta)\sigma_{k+1}, 
 \end{eqnarray*}
where the 2nd line follows from the triangle inequality and the last uses Theorem \ref{thm:rsvd} and the result from \cite{eckart1936approximation} that the best $k$ rank approximation of a matrix is given by its $k$ largest singular values and vectors with spectral residual $\sigma_{k+1}$. With similar remarks
\begin{eqnarray*} 
\lefteqn{\mathbb{E}\|T_{k}(\Ym) - \hat{T}_{k}(\Ym)\|_F^2}\\
  &=& \mathbb{E}\|T_{k}(\Ym) - \Ym + \Ym - \hat{T}_{k}(\Ym)\|_F^2  \\
  &\leq& \mathbb{E}(\|T_{k}(\Ym) - \Ym\|_F^2 + \|\Ym - \hat{T}_{k}(\Ym)\|_F^2)\\
  &\leq& \|\sigmav_{>k}\|_2^2 + 2k \left(1 + \frac{k}{k-1}\right)^{\frac{1}{2q+1}} \|\sigmav_{>k}\|_{(2q+1)},
 \end{eqnarray*}
where the final line uses Theorem \ref{thm:froSVDError}.

Now we can write $|f_\lambda(\Zm) - f_\lambda(\hat{\Zm})|$ as follows: 
\begin{align*}
&\leq \frac{1}{2} \|P_\omega(\Zm) - P_\omega(\hat{\Zm})\|_F^2 + \lambda \left| (\|\Zm\|_* - \|\hat{\Zm}\|_*)\right| \\ 
&\leq \frac{1}{2} \|P_\omega(\Zm) - P_\omega(\hat{\Zm})\|_F^2 + \lambda \|\Zm- \hat{\Zm}\|_* \\
&\leq \frac{1}{2} \|\Zm - \hat{\Zm}\|_F^2 + \lambda \sqrt{k} \|\Zm- \hat{\Zm}\|_F \\
&= \frac{1}{2} \|S_\lambda(\Ym) - \hat{S_\lambda}(\Ym)\|_F^2 + \lambda \sqrt{k} \|S_\lambda(\Ym)- \hat{S_\lambda}(\Ym)\|_F \\
&\leq \frac{1}{2} \|T_{k}(\Ym) - \hat{T}_{k}(\Ym)\|_F^2 + \lambda \sqrt{k} \|T_{k}(\Ym)- \hat{T}_{k}(\Ym)\|_F \\
&\leq  \frac{1}{2} \|T_{k}(\Ym) - \hat{T}_{k}(\Ym)\|_F^2 + \lambda k \|T_{k}(\Ym)- \hat{T}_{k}(\Ym)\|_2, 
\end{align*}
where the 5th line uses Lemma \ref{lem:normSoft} and we assume all singular values after $k$ are zero in the soft thresholded SVDs. If we take expectations of the above then
\begin{eqnarray*} 
&& \mathbb{E}|f_\lambda(\Zm) - f_\lambda(\hat{\Zm})| \leq  \lambda k (1 + \theta)\sigma_{k+1} + \\
&&  \quad \frac{1}{2} \|\sigmav_{>k}\|_2^2 + k \left(1 + \frac{k}{k-1}\right)^{\frac{1}{2q+1}} \|\sigmav_{>k}\|_{(2q+1)}, 
\end{eqnarray*} 
which is the required result. 
\end{proof}

\bibliographystyle{abbrv}
\bibliography{references}

\end{document}